\documentclass{article}

\usepackage[english]{babel}

\usepackage[letterpaper,top=2cm,bottom=2cm,left=3cm,right=3cm,marginparwidth=1.75cm]{geometry}

\usepackage{amsmath,amsfonts,amssymb,amsthm}
\usepackage{graphicx}
\usepackage[colorlinks=true, allcolors=blue]{hyperref}
\usepackage{authblk}

\usepackage{physics}
\usepackage{algorithm}
\usepackage{algpseudocode}

\usepackage{subfigure}
\usepackage{booktabs}
\usepackage{tabularx}
\usepackage{color}

\newtheorem{prop}{Proposition}

\newcommand{\T}{\mathsf{T}}

\title{Efficient Numerical Algorithm for Large-Scale Damped Natural Gradient Descent}
\author[1]{Yixiao Chen\thanks{yixiaoc@princeton.edu}}
\author[2]{Hao Xie}
\author[3]{Han Wang\thanks{wang\_han@iapcm.ac.cn}}
\affil[1]{\small Program in Applied and Computational Mathematics, Princeton University}
\affil[2]{\small Beijing National Laboratory for Condensed Matter Physics and Institute of Physics, Chinese Academy of Sciences}
\affil[3]{\small Laboratory of Computational Physics, Institute of Applied Physics and Computational Mathematics}

\begin{document}
\maketitle

\begin{abstract}
We propose a new algorithm for efficiently solving the damped Fisher matrix in large-scale scenarios where the number of parameters significantly exceeds the number of available samples. 
This problem is fundamental for natural gradient descent and stochastic reconfiguration.
Our algorithm is based on Cholesky decomposition and is generally applicable.
Benchmark results show that the algorithm is significantly faster than existing methods.

\end{abstract}

\section{Introduction}

Natural gradient descent \cite{amari1998natural, martens2020new} is a fundamental optimization technique widely employed in the field of machine learning. 
Its quantum counterpart, known as stochastic reconfiguration \cite{sorella1998green, sorella2005wave}, holds paramount importance in variational quantum Monte Carlo methods. 
However, when applied to large-scale problems, such as training neural networks, a significant bottleneck emerges due to the computational burden of inverting the Fisher information matrix.
Although approximations like KFAC \cite{martens2015optimizing} have been introduced to mitigate this burden, they often fall short of replicating the performance of the exact method.

In large-scale scenarios, where the number of samples is typically much smaller than the number of model parameters, a damping term becomes essential.
In this paper, we propose a fast algorithm for inverting the damped Fisher information matrix, based on Cholesky decomposition. 
This algorithm is designed for GPU implementation and can be easily parallelized, promising to significantly improve the scalability and performance of natural gradient descent and stochastic reconfiguration.

\section{Results}

Our objective is to find a solution, $x$, to the following linear equation:
\begin{equation}
    \label{eq:fisher}
    \left( S^\T S + \lambda I \right) x = v
\end{equation}
In this equation $S$ is a $n \times m$ matrix, while $x$ and $v$ are both $m$-dimensional vectors.
The parameter $\lambda$ determines the damping strength and $I$ represents the $m \times m$ identity matrix.
In the context of natural gradient descent, $S$ is the (scaled) score matrix, defined as $S_{ij} = \frac{1}{\sqrt{n}} \frac{\partial \log P_\theta(x_i)}{\partial \theta_j}$, where $P_\theta(x_i)$ is the model's predicted probability of sample $x_i$ and $\theta_j$ is the $j$-th parameter of the model. $S^\T S$ yields the estimated Fisher information matrix.
Correspondingly, $v$ is the gradient of the loss function $L$ with respect to the parameters $\theta$, $v_j = \frac{\partial L}  {\partial \theta_j}$.
We are primarily interested in scenarios where $m \gg n$, meaning that the number of parameters is significantly larger than the number of samples.
This results in $S^\T$ being a tall-and-skinny matrix.

\begin{algorithm}[htbp]
    \caption{Cholesky Solve of Damped Fisher}
    \label{alg:chol}
    \begin{algorithmic}[1]
        \Require $S$, $v$, $\lambda$
        \Ensure $x$ that satisfies $\left( S^\T S + \lambda I \right) x = v$

        \State $W \gets S S^\T + \lambda \tilde{I}$ 
        \Comment{$W$ is $n \times n$, $\tilde{I}$ is $n \times n$ identity}

        \State $L \gets \text{Chol}\left( W \right)$ 
        \Comment{Cholesky decomposition, $L$ is $n \times n$ lower triangular}

        \State $Q \gets L^{-1}S$ \label{alg:chol:q}
        \Comment{$Q$ is $n \times m$}

        \State $x \gets \frac{1}{\lambda} \left( v - Q^\T Q v \right)$ \label{alg:chol:x}
        \Comment{$Q$ can be inlined to further reduce cost}
        \State \Return $x$
    \end{algorithmic}
\end{algorithm}

We propose Algorithm~\ref{alg:chol} to solve Eq.~\ref{eq:fisher}.
The correctness of the algorithm is straightforward to verify, and a proof can be found Appendix~\ref{app:proof}.
The computational complexity of the algorithm is $\order{n^3 + n^2m}$, which is determined by the Cholesky decomposition and the matrix multiplication. 
Compared to the naive method of directly inverting the matrix ($\order{m^3}$), our proposed algorithm is significantly faster when $m \gg n$.
Moreover, The memory requirement is reduced from $\order{m^2}$ to $\order{nm}$.
The algorithm can be easily parallelized, and the Cholesky decomposition can be efficiently implemented on GPU. 
We note that in practical implementation, the computation of $Q$ (line~\ref{alg:chol:q}) should be inlined into the calculation of $x$ (line~\ref{alg:chol:x}) to further reduce computational cost. 
The resulted $Q^\T Q v = S^\T L^{-\T} L^{-1} S v $ can be efficiently computed from right to left with triangular solve.

We implemented the algorithm in JAX and conducted tests on a single NVIDIA A100 GPU with 80\,GB of memory. 
We evaluated the algorithm's performance on problems with $m \sim 10^6$ parameters and $n \sim 10^3$ samples, which is beyond the capability of the naive inversion method.
Our benchmarking compared the proposed algorithm (``chol'') to two SVD-based methods (see Appendix~\ref{app:svd}): 
one using the CUDA kernel \verb|gesvda| (labeled ``svda''), and the other utilizing the fast SVD algorithm for tall-and-skinny matrices via solving the eigenproblem of the $n \times n$ matrix $S S^\T$ (labeled ``eigh''). 
The ``eigh'' SVD method is previously the fastest method in our experience. 

The benchmark results, shown in Fig.~\ref{fig:benchmark} illustrate the algorithm's consistent improvement over the two SVD-based methods.  
The speedup is particularly pronounced when the $n/m$ ratio is larger. 
The algorithm incurs minimal overhead and scales quadratically with $n$ and linearly with $m$, aligning with the theoretical complexity analysis for $m \gg n$.
The full data can be found in Appendix~\ref{app:data}.

\begin{figure}[htbp]
    \centering
    \subfigure[increasing $n$ with fixed $m = 10^6$]{\label{fig:bench_batch}\includegraphics[width=0.48\textwidth]{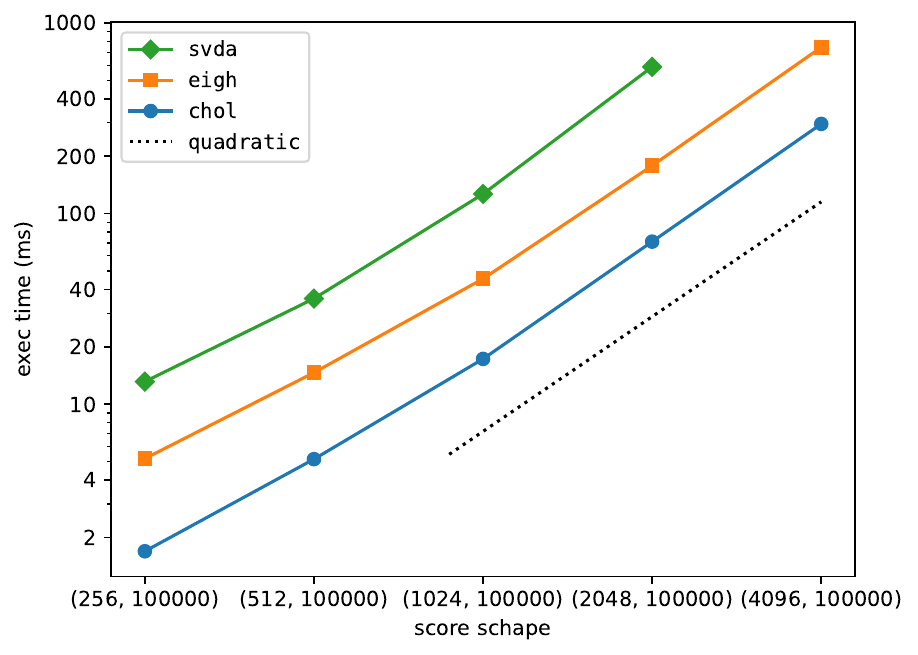}}
    \subfigure[increasing $m$ with fixed $n = 2048$]{\label{fig:bench_param}\includegraphics[width=0.48\textwidth]{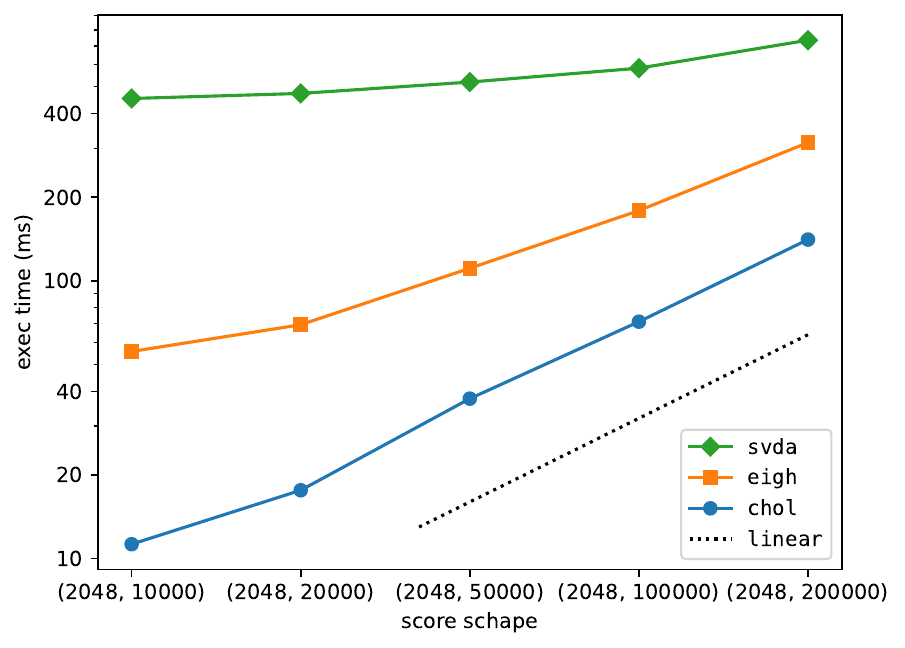}}
    \caption{Benchmark of the proposed algorithm (``chol'') against two SVD-based methods (``eigh'' and ``svda''),  with increasing samples ($n$) or parameters ($m$). Dotted lines represent the ideal scaling. Please refer to the main text for details.}
    \label{fig:benchmark}
\end{figure}

\section{Relation to other methods}

There are multiple other problems and methods that share the same structure as Eq.~\ref{eq:fisher}, 
such as damped least-squares (a.k.a. Levenberg-Marquardt algorithm \cite{gavin2019levenberg}) in optimization, ridge regression \cite{hoerl1970ridge} in statistics, and stochastic reconfiguration (SR) \cite{sorella2005wave} in variational quantum Monte Carlo. 
Our proposed algorithm can naturally be applied to these problems as well, and it offers substantial speedup advantages when $m \gg n$.

Particularly, in the context of SR, $S$ needs to be properly centered due to the unnormalized nature of the wave function, $S = \frac{1}{\sqrt{n}} \left( O - \overline{O} \right)$, 
where $O$ is the derivative of the logarithm of the wave function $O_{ij} = \frac{\partial \log \psi_\theta(x_i)}{\partial \theta_j} $ and $\overline{O}$ denotes the average of $O$ over the samples. 
In cases where the wave function is complex, $S$ becomes a complex matrix, and the transpose operation ($S^\T$) in Eq.~\ref{eq:fisher} is replaced by the Hermitian conjugate ($S^\dagger$). 
There are two variations in the SR algorithm in this complex case, differing in the definition of Fisher matrix: 
one uses the full complex matrix, $F = S^\dagger S$, while the other uses only the real part, $F = \real \left[S^\dagger S\right]$. 
The latter is more commonly employed in practice. 
Our algorithm can be readily adapted to both versions.
For the complex version, we can replace all transposes in our algorithm with Hermitian conjugates.
For the real part version, the score matrix $S$ can be replaced by the concatenation of its real and imaginary parts in the sample ($n$) dimension, as follows: $S \gets \text{Concat}\left[ \real(S), \imaginary(S) \right] $, while the rest of the algorithm remains unchanged.

It's worth noting that there are existing papers addressing large-scale SR, particularly for neural network wavefunctions, such as \cite{chen2023efficient} and \cite{rende2023simple}. 
These methods rely on the least-square structure of the SR procedure, meaning that the gradient $v$ is a linear combination of the rows of $S$, i.e., $v = S^\T f$. 
This requirement limits the choice of the loss function and prevents the use of regularization. 
In contrast, our algorithm does not have this limitation and can be applied to any loss function, including those used in Wasserstein quantum Monte Carlo \cite{neklyudov2023wasserstein}. 
When applied to SR, our algorithm is nearly identical to the one proposed in \cite{rende2023simple} (if Cholesky solve is used there), and the computational cost is almost the same.
The connection between the methods is discussed in Appendix~\ref{app:equiv}.
Our algorithm can also share the same parallelization strategy as illustrated in the supplement material of \cite{rende2023simple}.

Iterative methods, such as conjugate gradient descent, can efficiently solve Equation \ref{eq:fisher}. 
These methods typically scale linearly with both $n$ and $m$, but the number of iterations increase significantly when the matrix is ill-conditioned. 
This leads to slow convergence and higher computational costs. 
Our algorithm, on the other hand, is non-iterative and generally avoids such issues.

\appendix
\section{Proof of correctness}
\label{app:proof}
\begin{prop}Given matrix $S\in \mathbb R^{n\times m}$ and a parameter $\lambda \in\mathbb R^{+}$ being a large enough positive number so that $\left( S^\T S + \lambda I \right)$ is inversible, then the $x$ given by Algorithm~\ref{alg:chol} solves Equation~\eqref{eq:fisher}.
\end{prop}
\begin{proof}
    We denote that $x^\ast = \left( S^\T S + \lambda I \right)^{-1}v$, then 
\begin{align*}
    \left( S^\T S + \lambda I \right) x^\ast = v 
    \:\Rightarrow\:
    S S^\T S x^\ast + \lambda S x^\ast = Sv 
    \:\Rightarrow\:
    \left(S S^\T + \lambda \tilde I \right) S x^\ast = Sv 
\end{align*}
The step 1 and 2 of Algorithm~\ref{alg:chol} yields $\left(S S^\T + \lambda \tilde I \right) = LL^\T$, and step 3 gives $S = LQ$, then
\begin{align*}
    \left(S S^\T + \lambda \tilde I \right) S x^\ast = Sv 
    \:\Rightarrow\:
    L L^\T LQ x^\ast = LQ v
    \:\Rightarrow\:
    Q^\T L^\T LQ x^\ast = Q^\T Q v
    \:\Rightarrow\:
    S^\T S x^\ast = Q^\T Q v
\end{align*}
The output $x$ of Algorithm~\ref{alg:chol} satisfies
$\lambda x = v -  Q^\T Q v$, thus
\begin{equation}\label{eq:ap:1}
    S^\T S x^\ast + \lambda x = v.
\end{equation}
Comparing \eqref{eq:ap:1} with \eqref{eq:fisher} gives $x = x^\ast$ due to the positiveness of $\lambda$.
\end{proof}

\section{Connection between the methods}
\label{app:equiv}
By our method, the solution  $x$ is computed by
$\frac 1\lambda(v - Q^\T Q v)$. Noticing that 
$ Q^\T Q v = S^\T L^{-\T} L^{-1} S v = S^\T (SS^\T +\lambda\tilde I)^{-1} Sv$, when the right hand side has the structure $v = S^\T f$,  we have
\begin{align}\label{eq:ap:ours}
    x = \frac 1\lambda S^\T f  - \frac 1\lambda S^\T (SS^\T +\lambda\tilde I)^{-1} S S^\T f
\end{align}
We denote the method proposed by~\cite{rende2023simple}, 
\begin{align}\label{eq:ap:rende}
    x_{\mathrm{rvb}} = S^\T (SS^\T +\lambda\tilde I)^{-1} f
\end{align}
The apparently different expressions~\eqref{eq:ap:ours} and \eqref{eq:ap:rende} are actually equivalent. 
We consider the difference between $x_{\mathrm{rvb}}$ and $x$,
\begin{align*}
    x_{\mathrm{rvb}} - x  &= 
    S^\T (SS^\T +\lambda\tilde I)^{-1} f - 
    \frac 1\lambda S^\T f  +
    \frac 1\lambda S^\T (SS^\T +\lambda\tilde I)^{-1} S S^\T f \\
    &=
    S^\T (SS^\T +\lambda\tilde I)^{-1} f - 
    \frac 1\lambda S^\T f  +
    \frac 1\lambda S^\T (SS^\T +\lambda\tilde I)^{-1} (S S^\T + \lambda\tilde I)  f
    - \frac 1\lambda S^\T (SS^\T +\lambda\tilde I)^{-1} \lambda\tilde I f \\
    &= 0
\end{align*}

\section{SVD-based methods}
\label{app:svd}
We provide an outline of the SVD-based methods for solving Eq.~\ref{eq:fisher} that we used in our benchmark. 
Suppose we have the (thin) SVD of $S$ as $S = U \Sigma V^\T$, where $U$ and $V$ are $n \times n$ and $m \times n$ orthogonal matrices and $\Sigma$ is a $n \times n$ diagonal matrix with non-negative entries. 
Then the solution to Eq.~\ref{eq:fisher} is given by
\begin{equation}
    \label{eq:ap:svd}
    x = V \left( \Sigma^2 + \lambda \tilde{I} \right)^{-1} V^\T v + \frac{1}{\lambda} \left( v - V V^\T v \right)
\end{equation}
The correctness of this method can be verified by direct substitution, noting that $S^\T S = V \Sigma^2 V^\T $, $V^\T V = \tilde{I}$ and $V V^\T$ is a projection matrix.

The two SVD methods mentioned in the main text differ in the way of computing the SVD of $S$.
For the ``svda'' method, we solve the SVD by calling the CUDA kernel \verb|gesvda| from PyTorch. 
For the ``eigh'' method, we first compute the eigenvalue decomposition of $S S^\T$ as $S S^\T = U \Sigma^2 U^\T$, and then finish the SVD by $V = S^\T U \Sigma^{-1}$.

\section{Benchmark data}
\label{app:data}
We provide in Table~\ref{tab:benchmark} the detailed benchmark data for Fig.~\ref{fig:benchmark}. 
The ``svda'' method is not available for shape (4096, 100000) due to memory limit. 
The difference between two (2048, 100000) rows is due to runtime fluctuation.

\begin{table}[htbp]
    \label{tab:benchmark}
    \caption{Benchmark results of shape ($n$, $m$). Time is in milliseconds }
    \centering
    \newcolumntype{R}{>{\raggedleft\arraybackslash}X}
    \begin{tabularx}{0.6\textwidth}{l|RRR}
        \toprule
        shape & chol & eigh & svda \\
        \midrule
        (256, 100000) & 1.69 & 5.18 & 13.14 \\
        (512, 100000) & 5.15 & 14.64 & 35.82 \\
        (1024, 100000) & 17.28 & 45.51 & 126.65 \\
        (2048, 100000) & 71.25 & 178.27 & 588.04 \\
        (4096, 100000) & 295.20 & 745.17 & N/A \\
                \midrule
        (2048, 10000) & 11.27 & 55.69 & 453.27 \\
        (2048, 20000) & 17.63 & 69.49 & 472.67 \\
        (2048, 50000) & 37.67 & 110.99 & 519.34 \\
        (2048, 100000) & 71.27 & 179.01 & 582.82 \\
        (2048, 200000) & 140.79 & 314.47 & 734.84 \\
        \bottomrule
    \end{tabularx}    
\end{table}

\bibliographystyle{alpha}
\bibliography{ref}

\end{document}